\documentclass[journal]{IEEEtran}
\usepackage{cite}

\usepackage[numbers,sort&compress]{natbib}

\ifCLASSINFOpdf
 \usepackage[pdftex]{graphicx}
\else
\fi
\usepackage{amsthm,amsmath,amssymb}
\newtheorem{mydef}{Definition}
\newtheorem{mytheo}{Theorem}
\usepackage{booktabs}
\usepackage{multirow}

\usepackage{algorithm}
\usepackage{algorithmic}
\begin{document}
%
\title{Learning Invariable Semantical Representation from Language for Extensible Policy Generalization}
%
%
%

\author{Yihan Li, Jinsheng Ren, Tianrun Xu, Tianren Zhang, Haichuan Gao, and Feng Chen, ~\IEEEmembership{Member,~IEEE,}
\thanks{This work was supported in part by the National Natural Science Foundation of China under Grant 61836004, and in part by the Tsinghua-Guoqiang research program under Grant 2019GQG0006. Corresponding author: Feng Chen}
\thanks{Y. Li, J. Ren, T. Zhang, H. Gao and F. Chen are with the Department of
	Automation, Tsinghua University, Beijing 100086, China, with the
	Beijing Innovation Center for Future Chip, Beijing 100086, China,
	and with the LSBDPA Beijing Key Laboratory, Beijing 100084, China (e-mail: lyh19@mails.tsinghua.edu.cn; rjs17@mails.tsinghua.edu.cn; zhang-tr19@mails.tsinghua.edu.cn; ghc18@mails.tsinghua.edu.cn; chenfeng@mail.tsinghua.edu.cn)}
\thanks{Tianrun Xu is with College of Information, North China University of Technology, 100144, Beijing, China (e-mail: 18151010710@mail.ncut.edu.cn)}
}

%
%

\markboth{Under Review}%
{Shell \MakeLowercase{\textit{et al.}}: Bare Demo of IEEEtran.cls for IEEE Journals}
%



\maketitle

\begin{abstract}
	Recently, incorporating natural language instructions into reinforcement learning (RL) to learn semantically meaningful representations and foster generalization has caught many concerns. However, the semantical information in language instructions is usually entangled with task-specific state information, which hampers the learning of semantically invariant and reusable representations. In this paper, we propose a method to learn such representations called \textit{element randomization}, which extracts task-relevant but environment-agnostic semantics from instructions using a set of environments with randomized elements, e.g., topological structures or textures, yet the same language instruction. We theoretically prove the feasibility of learning semantically invariant representations through randomization. In practice, we accordingly develop a hierarchy of policies, where a high-level policy is designed to modulate the behavior of a goal-conditioned low-level policy by proposing subgoals as semantically invariant representations. Experiments on challenging long-horizon tasks show that (1) our low-level policy reliably generalizes to tasks against environment changes; (2) our hierarchical policy exhibits extensible generalization in unseen new tasks that can be decomposed into several solvable sub-tasks; and (3) by storing and replaying language trajectories as succinct policy representations, the agent can complete tasks in a one-shot fashion, i.e., once one successful trajectory has been attained.
\end{abstract}


%
\IEEEpeerreviewmaketitle

\section{Introduction}
%
%
%
%
\IEEEPARstart{R}{einforcement} learning (RL) \cite{sutton1988reinforcement} has achieved remarkable results in many applications \cite{,mnih2013playing,silver2017mastering}. However, current RL algorithms often suffer from poor generalization ability. One of the prevailing ideas towards this problem is that current RL methods lack a general and compact representation (abstraction) to express shared and reusable knowledge among different environments with similar semantics \cite{oh2017zero,zhang2017a}. On the other hand, natural language serves exactly as a general representation with interpretability. Therefore, a growing number of research begins to focus on integrating natural language into policy learning, namely language-conditional RL \cite{,DBLP:conf/ijcai/LuketinaNFFAGWR19}. The final goal is to empower the agent with the ability of extracting semantical information in language instructions and learning generalizable policies \cite{jiang2019language,hu2019hierarchical}.

However, current methods can hardly learn reusable policies that can adapt to new environments. On the contrary, their learned policies are usually environment-specific which cannot always be reused. For example, a task that requires the agent to pick up a ball will lead the agent to memorize several successful trajectories instead of extracting the semantical concept of ``ball''. As a consequence, the learned representation is unstable and correlates with the environment, which hinders policy generalization.

The main reason for the aforementioned problem is that the language is usually applied to an end-to-end policy learning framework \cite{jiang2019language,hu2019hierarchical,devin2019plan}, which directly maps language instructions together with states to actions without explicitly incentivizing the development of invariant representations. An invariant representation is invariant to environment changes that still preserve the semantics of the task (e.g., ``pick up the green ball''), thus is generalizable and reusable across different environments. Hence, two major issues need to be addressed to learn such representations: (1) the motivation for semantical invariance of the representations, and (2) a carefully-designed policy structure that contains invariant components to carry these invariant representations.

\begin{figure}
	\centering
	\includegraphics[scale = 0.25]{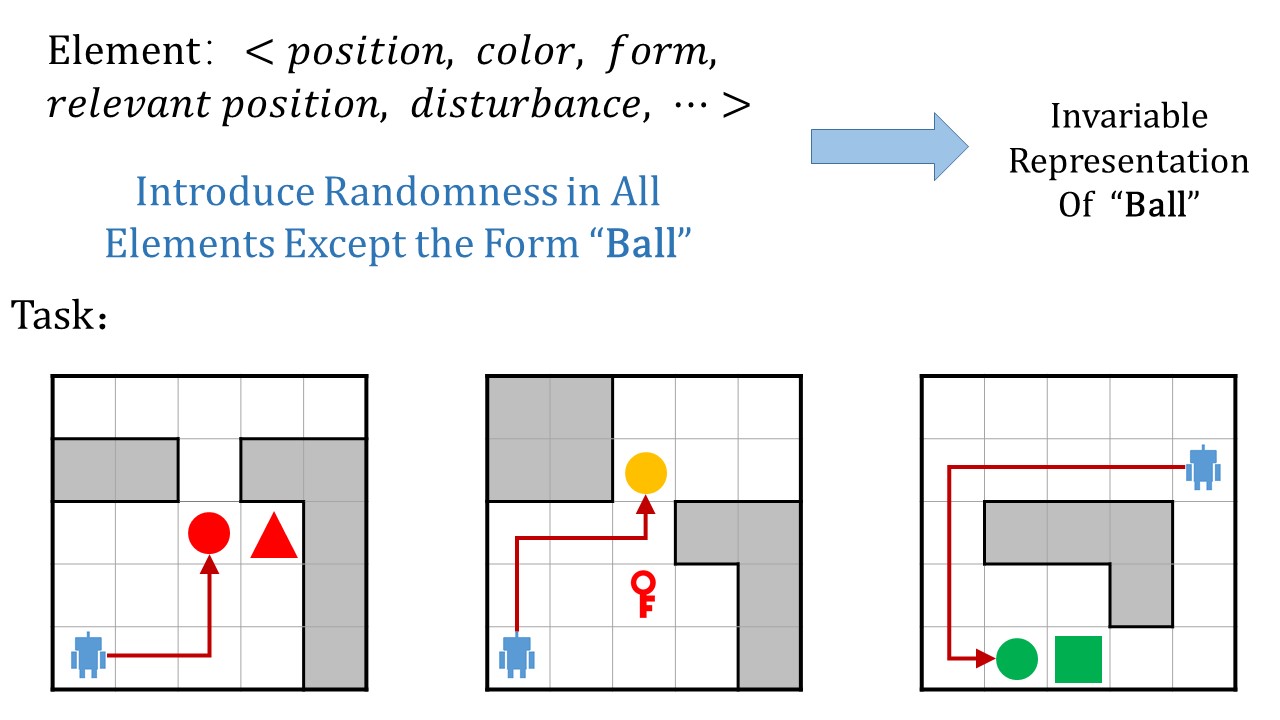}
	\caption{We introduce randomness into elements except the goal to decouple the goal element from the environment. With making completing task consist with extracting semantical invariants, we can get reusable representations which can resist the change of the environment.}
	\label{FIG:1}
\end{figure}

In this paper, we propose a method termed \emph{element randomization} to learn invariant representations. Our main idea is to introduce randomness in environment-specific elements (or components), e.g., the topological structure of the state space or texture of the objects, to facilitate the agent to extract semantical commonalities that can resist environment changes. This idea has also been used by approaches in other domains, e.g., Domain Randomization (DR) \cite{,tobin2017domain,DBLP:conf/iros/GaoYS0020}. The main difference is that we introduce randomness at the \emph{task level}, meaning that anything that does not correlate with the core semantics of the task may be randomized (e.g., in a maze navigation task, not only the texture of the maze can be randomized, but also the structure of the maze itself can be randomized as long as it does not alter the semantics of the task ``find the exit''), in contrast with prior methods that only perform visual-level randomization. Our approach thus provides motivation for decoupling the entangled elements and ensures invariance (see Fig. \ref{FIG:1}). We also give theoretical justifications, showing that randomizing elements can indeed result in invariant representations.

In practice, we design an adaptive model structure to extract semantically invariant representations and learn the corresponding semantically invariant policy.
Concretely, we construct a two-level goal-conditioned hierarchical policy network. Our policy network consists of a high-level policy that receives language instructions and generates subgoals as semantically invariant representations, and the low-level policy executes atomic actions in accord with the proposed subgoals. As the subgoals represent stable language-conditioned representations, they can be reversely translated into their language form. That means tasks can be explored and recorded as language trajectories. Therefore, we introduce an external memory to record the trajectories with extensible lengths. With the external memory, the agent can solve a new task with arbitrary length by exploring through selecting subgoals and exploiting by replaying the successful trajectories.

To demonstrate the superiority of our method, we conduct experiments on challenging tasks with long horizons based on the BabyAI platform \cite{chevalier-boisvert2018babyai}. Experimental results validate the efficacy of our method, showing that (1) our low-level policy reliably generalizes to tasks against environment changes; (2) our hierarchical policy exhibits extensible generalization in unseen new tasks that can be decomposed into several sub-tasks solvable by the low level; and (3) by storing and replaying language trajectories using the external memory, the agent can accomplish the task in a one-shot fashion, i.e., once one successful trajectory has been attained.

In short, our contributions are as follows:

\begin{enumerate}
	\itemsep=0pt
	\item We propose a new language-conditional policy learning paradigm, which extracts semantically invariant representasions by a novel element randomization method.
	
	\item We build an adaptive hierarchical network for simultaneous language comprehending and task executing. We also add an external memory to record abstracted language trajectories for exploring and solving unseen complex tasks.
	
	\item We both theoretically and empirically demonstrate the efficacy our model. Experimental results on several difficult tasks show the superiority of our agent compared with several strong baselines.
\end{enumerate}

\begin{figure*}
	\centering
	\includegraphics[scale=.45]{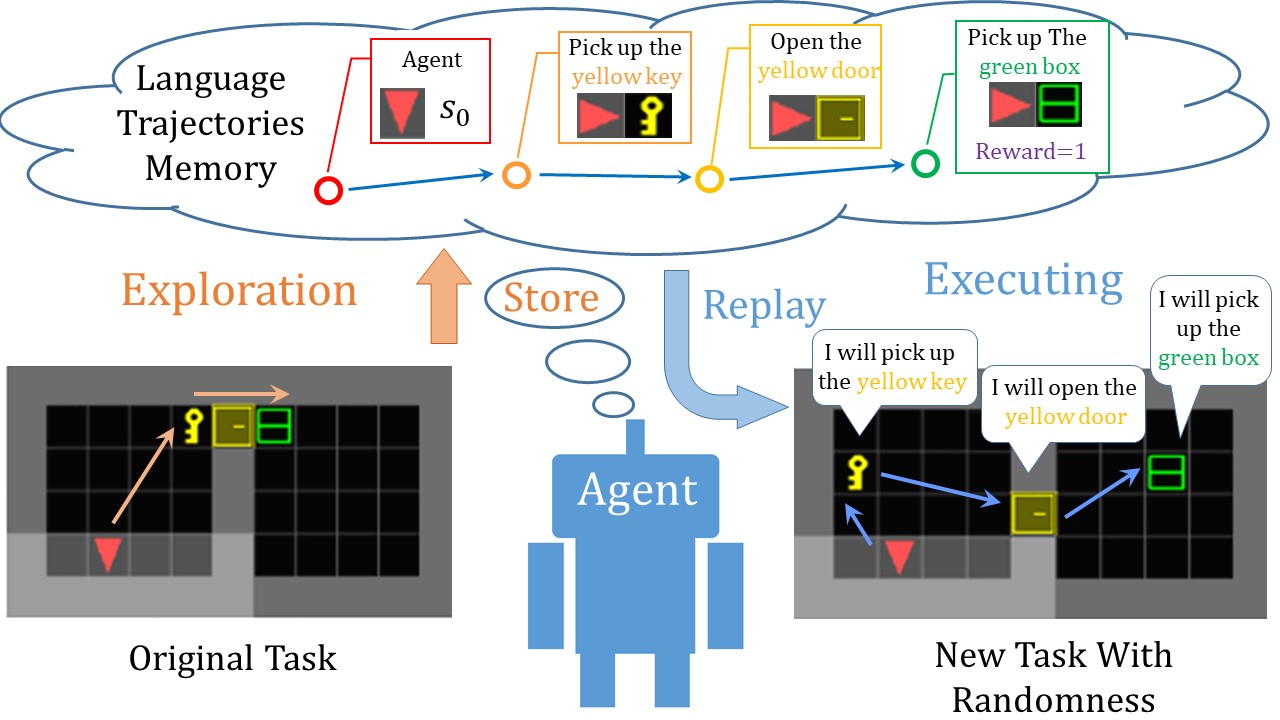}

	\caption{We let the agent learning invariable semantical representations and corresponding policies. When facing a task, the agent can explore the environment with low-dimension subgoal space instead of the original state space. Meanwhile, the trajectories can be stored as language form and be replayed after collecting the final reward and the agent can overcome the randomness in the environment.}
	\label{FIG:2}
\end{figure*}

\section{Related Work}

\subsection{RL Methods with Natural Language Condition}

Recently, many researchers begin to focus on learning a framework with the power of natural language, such as \cite{hu2019hierarchical,zhong2019rtfm,xiong2018hierarchical,jiang2019language,bahdanau2018learning,hermann2017grounded} and \cite{,branavan2011learning}. Language can be used in many ways, such as the tool of knowledge transfer \cite{,narasimhan2015language,mei2016listen}, the reward generator \cite{ho2016generative,goyal2019using,wang2019reinforced}, the representation of state or action space \cite{yuan2018counting}, the communicating domain knowledge or task-independent corpora \cite{,oh2017zero,narasimhan2018grounding}. In this paper, the most related work is to use language to describe the task as inputting instructions, hoping language can guide the agent to complete the task and generalize well in some domains. However, comprehending language is such a difficult problem itself. They aiming to use language to help reinforcement learning gaining more cumulative reward. But just using language as additional information, the agent may leverage language as trivial coding of a task \cite{lake2018generalization,liu2020compositional}. Because the agent can learn to comprehend the part of the semantical information, but it is not sure that the semantics is the stable and general one that we want, not to mention that the agent can automatically extract and reuse it. 

The difference is that our work introduces an invariable subgoal space between language and original actions, which is specially trained to represent object-oriented sub-policies with semantics.
These sub-policies is general and reusable in the whole environment with randomness. That means our agent can explore in the low-level subgoal space rapidly and generalize to new tasks with a high probability.

\subsection{Element-Oriented Compositional RL }
Some researchers analyzed RL tasks by disassembling elements in tasks and construct corresponding policies \cite{,singh1992transfer,meer2020exploiting,garnelo2016towards,zamani2018deep}. Some of them also use symbolic coding as task representations to further leverage these element-oriented policies. These methods have achieved much success in compositional generalization problems. However, to some degree, these methods focus more on task-specific elements (such as a fixed special point in a room), which cannot be reused in all the tasks. It will limit the generalization ability.

Different from these works, our method focus on the elements which are interactive in the whole environment, and at any time they are task-agnostic. That is, our sub-policies are object-oriented and can be reused in any task. That guarantees the sub-policy can be described by language without any ambiguity. So that the combination of sub-policies can be recorded as language trajectories and replayed stably. 

\subsection{HRL Methods for Generalization}

Some hierarchical reinforcement learning (HRL) researchers focus on generalization problems, using the measure of multi-task learning or graph representation. Such as \cite{,andreas2017modular} building policy sketches to guide the agent to complete tasks. \cite{,frans2017meta} construct an HRL method with meta parameters as a high-level abstraction. \cite{,sohn2018hierarchical} give the agent a structured subtask graph to represent the relationship of tasks. However, to solve complex tasks and generalize to more tasks, previous works need manual prior knowledge of tasks information more and less. For instance, framework in \cite{,andreas2017modular} need a policy sketch to describe task precisely, which needs human transform tasks into a special form that can be comprehended by the agent. These kinds of methods need not the only representation of the task, but the way that how they are executed in a high-level step. When a new task cannot be expressed by the prior, such as adding an extra high-level step or facing a new task composed of subtasks but the sequence is unseen, their agents can hardly complete the task. 

Different from their methods, our work takes natural language as a task representation, which is general prior and can describe almost any task. A general representation means that the agent can complete most of the tasks following the instructions, even unseen environment. So our work can adapt to new tasks which are even more complex and difficult than the training tasks.

\section{Learning Invariable Semantical Representations by Element-randomization}

In this section, to verify the effectiveness of our method, we provide a qualitative mathematical analysis to show that introducing randomness into elements can decouple the goal element with others and make completing tasks equivalent to extracting invariable representations of the goal elements.  

\subsection{Problem Statement}
We consider a finite-horizon, goal-conditional \textit{Markov Decision Process} (MDP) with language as instruction, which is defined as $<\mathcal{S,G,I,A,P,R}, \rho, \gamma >$, where $\mathcal{S}$ is the state set, $\mathcal{G}$ is the invariable subgoal set, $\mathcal{I}$ is the language instruction set, $\rho (I)$ is the initial distribution of instruction $I \in \mathcal{I}$, $\mathcal{A}$ is the action set, $\mathcal{P:S \times I \times A \times S \rightarrow} [0,1)$ is the state transition function representing for the probability from state and instruction to next state, $\mathcal{R:S \times A \rightarrow} \mathbb{R}$ is the reward function, and $\gamma \in [0,1)$ is the discount factor. The objective of reinforcement learning is to learning a policy $\pi(a_t|s_t, I)$. Every instruction can stand for a class of similar tasks. The RL framework is built on the whole task set $\mathcal{T}$. 

\subsection{Methodology Analysis of Equivalence}
To extract invariable semantical representations, we design the task with element-randomization method. Our main idea is that the agent should extract the representations by learning in RL tasks. Here we will give the proof that maximizing the cumulative return with element-randomization is equivalent to maximizing the occurring probabilities of the goal elements. Connecting the invariable subgoals with these elements by RL policy will endow these subgoals fixed meaning and make them semantical representations.

\begin{mydef}
	Let task $T \in \mathcal{T}$. The objective function is maximizing the expected calculative return $G$ from all the tasks $T$.
	$$ J = \mathbb{E}_{T \in \mathcal{T}} [G]$$
\end{mydef}

Different from traditional RL process that maximizing reward from single task or several task, we define the objective function among a large set of tasks. In our setting, an instruction $I$ can represent a class of tasks $T$ with same goal and every task can be completed by several correct trajectories $\tau$. 

With element-randomization method, we have:
\begin{mytheo}
	Introducing randomness with maximizing the cumulative return $J$ is equivalent to maximizing the occurring probability of invariable goal element $\epsilon$ with the spare reward setting.
	
	$$
    \mathbb{E}_{T \in \mathcal{T}} [G]	= \mathbb{E} [\sum \limits_{\epsilon} \rho (I_\epsilon) \pi(\epsilon| I_\epsilon)]	
	$$
\end{mytheo}

\begin{proof}
	For every language instruction $I$, it will command agent to interact with a goal element $\epsilon$, so that the function can be rewritten as:
	\begin{equation}
	J = \mathbb{E}_{T \in \mathcal{T}} [\sum \limits_{\epsilon} \rho (I_\epsilon) G(I = I_\epsilon)] 
	\end{equation}
	For all the task $T$ of this instruction, there is:
	\begin{equation}
	J = \mathbb{E}_{T \in \mathcal{T}} [\sum \limits_{\epsilon} \rho (I_\epsilon) \sum \limits_{T} P(T|I_\epsilon)G(T|I_\epsilon)]
	\end{equation}
	where the $P(T|I_\epsilon)$ and $G(T|I_\epsilon)$ are the distribution and corresponding return of task $T$. The return can be further written as follow for all the trajectories $\tau \in T$ with policy $\pi(\tau)$:
    \begin{equation}
	G(T|I_\epsilon) = \sum \limits_{\tau} \pi(\tau)R(\tau|T,I_\epsilon)
    \end{equation}
	In our design, only when agent interacting with correct element will it gain reward $1$, otherwise $0$:
	\begin{equation}
	R(\tau|T,I_\epsilon) = \begin{cases} 
	1,  & \mbox{if } \epsilon \in \tau \\
	0, & \mbox{otherwise} 
	\end{cases} \label{func1}
	\end{equation}
	Here $\epsilon \in \tau$ means that agent correctly interact with the goal element in this trajectory. So the total objective function can be written as:
	\begin{equation}
	J = \mathbb{E}_{T \in \mathcal{T}} [\sum \limits_{\epsilon} \rho (I_\epsilon) \sum \limits_{T} P(T|I_\epsilon)\sum \limits_{\tau} \pi(\tau)R(\tau|T,I_\epsilon)] \label{func2}
	\end{equation}
	Then put equation (\ref{func1}) into (\ref{func2}) and explicitly extract invariable element $\epsilon$, we get
	\begin{align}
	\begin{split}
	J = \mathbb{E}_{T \in \mathcal{T}} [\sum \limits_{\epsilon} & \rho (I_\epsilon) \sum \limits_{T}  P(T|I_\epsilon)  \sum \limits_{\underbrace{\tau : \epsilon \in \tau}_{\mbox{\shortstack[c]{\fontsize{6.8pt}{\baselineskip}\selectfont successful \\ \fontsize{7.0pt}{\baselineskip}\selectfont trajectories }}}}\\  & \sum \limits_{s,a \in \tau} \underbrace{\pi^{L}(a|s, \epsilon)}_{\mbox{\shortstack[c]{\fontsize{6.8pt}{\baselineskip}\selectfont policy from invariants \\ \fontsize{7.0pt}{\baselineskip}\selectfont to action }} }
	 \underbrace{\pi^{H}(\epsilon | T, I_\epsilon)]}_{\mbox{\shortstack[c]{\fontsize{6.8pt}{\baselineskip}\selectfont policy from language \\ \fontsize{7.0pt}{\baselineskip}\selectfont to invariants }} }
	\end{split} \label{func3}
	\end{align}
	where the irrelevant element will be eliminated by $R = 0$ with sufficient exploration.

	When introducing randomness into task-agnostic elements, the goal element will be decoupled with others. As a result, getting summation of the probability of $\pi(\epsilon \in \tau)$ is equal to get the total probability of the appearing of $\epsilon$, that is:
	\begin{equation}
	J = \mathbb{E} [\sum \limits_{\epsilon} \rho (I_\epsilon) \pi(\epsilon| I_\epsilon)] 
	\end{equation}
\end{proof}

By the theory, we leverage the subgoal $g \in \mathcal{G}$ to represent the goal element $\epsilon$. As a result, the subgoals represents invariable semantics, which can be reused in any task unambiguous. The policies are also built according to the subgoal, which can be general policies among tasks.

\section{Building Hierarchical Semantical Invariants Learning Network}
In this section, we will show the main idea of our method and present our framework for training a two levels of hierarchical policy with the object-oriented subgoal space guiding by equation (\ref{func3}) as follow. 
\begin{enumerate}
	\itemsep=0pt
	\item The high-level policy receives language instructions and observation of the environment and chooses subgoal as $\pi^{H}(\epsilon | T, I_\epsilon)$.
	
	\item The low-level policy receives the subgoal and executes corresponding actions in stochastic environment as $\pi^{L}(a|s, \epsilon)$.
	
	\item For learning extensible policy with the help of language instructions, we will show that we make use of an augmented-memory to record language as abstracted trajectories for replaying in the unseen new tasks without retraining.
	
\end{enumerate}

\subsection{Training Low-Level Stable Object-Oriented Subgoal Executor Policy}
 
\begin{figure}
	\centering
	\includegraphics[scale=.32]{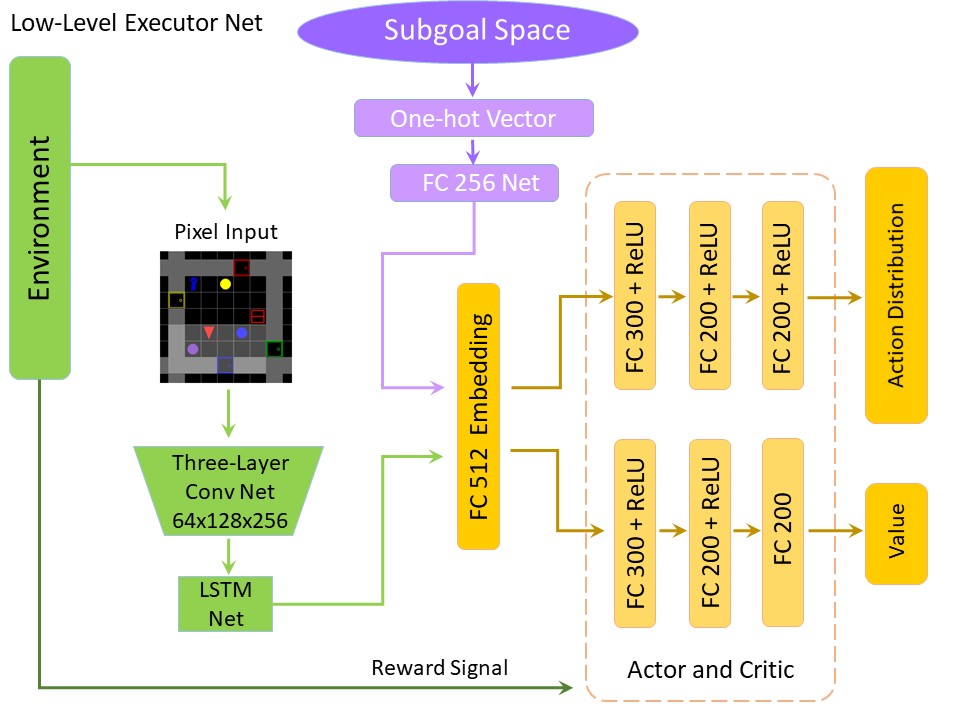}
	\caption{We build the low-level policy with multi-process A2C algorithm, which receives both the subgoal from the high-level policy and the pixel observation of the environment. We called it subgoal executor network (SEN).}
	\label{FIG:3}
\end{figure}

Here we show how to leverage element randomization method to build stable low-level policy, which we called subgoal executor network (SEN).

As we set the subgoals as semantical invariants, it requires the corresponding policies to adapt to the subgoals. That means these policies should be task-agnostic and general. To solve this problem, we design the low-level policy as a stable object-oriented policy, which focuses on specific objects or attribution and can be reused in any task.

However, to build a stable object-oriented policy, merely RL method cannot provide sufficient motivation. As we said above, we introduce the element-randomization method to make completing tasks equivalent to learning semantical invariants. Specifically, we use massive simple tasks that can be completed by only interacting with one object. Such as ``opening a door" or ``pick up a ball" with only one goal are all simple tasks. These simple tasks stand for the basic element composing the environment. They are quite few so that an ergodic sampling of them is acceptable and costs little.

\begin{algorithm}[h]
	\caption{Low-level Subgoal Executor Learning Algorithm}\label{algorithm}
	\begin{algorithmic}[1]
		\STATE Initialize multi-process actor parameters $\theta^i_a$ for $ i \in [1,n]$
		\STATE Initialize multi-process value parameters $\theta^i_v$ for $ i \in [1,n]$	
		\FOR {episodes in 1,M}
		
		\FOR {$i \in [1,n]$}
		
		\STATE Reset gradients: $d\theta^i_a$ and $d\theta^i_v$
		\STATE Synchronize thread-specific parameters
		\STATE Sample subgoal $g\in \mathcal{G}$ in uniform distribution 
		\REPEAT
		\STATE Perform $a_t$ according to policy $\pi(a_t|s_t,g)$
		\STATE Receive reward $r_t$ and new state $s_{t+1}$
		\STATE $t \leftarrow t+1$
		\UNTIL {terminal $s_T$ or $t-t_{start}==t_{max}$}
		\STATE Set
		$$R = \begin{cases} 
		0,  & \mbox{for terminal state } s_T \\
		V(s_t,\theta'_v), & \mbox{otherwise}
		\end{cases}$$
		\FOR {$j \in {t-1,\dots,t_{start}}$}
		\STATE $R \leftarrow r_ j + \gamma R$
		\STATE Accumulate gradients wrt $\theta'_a$
		$$
		d\theta_a \leftarrow d\theta_a+\nabla_{\theta_a'}\log\pi(a_j|s_j;\theta_a')(R_i-V(s_i;\theta_v'))
		$$
		\STATE Accumulate gradients wrt $\theta'_v$
		$$
		d\theta_v \leftarrow d\theta_v + \frac{\partial}{\partial \theta_v}(R_j-V(s_j;\theta_v'))^2
		$$
		\ENDFOR
		
		\ENDFOR
		\STATE Synchronize and update parameters

		\ENDFOR

	\end{algorithmic}

\end{algorithm}

Then we introduce other interference objects into these tasks. Actually, in a room, the goal object has a random position. And we use an unambiguous one-hot vector as a subgoal to express the final goal so that every dimension of the subgoal represents one fixed object. Each task has disturbances which are independent of the goal of the task, for example, task-independent objects. Only when the agent identifies the subgoal correctly, overcoming the random disturbances and interacts with the correct object, will it obtain a sparse reward which is discounted according to the steps it used. The task setting brings a constraint to the learning process besides RL motivation, forcing the agent to interact with the only object and ignore the disturbances. Then the agent will build a robust object-oriented policy that can be used in any other task and environment consisting of the same elements. 

Concretely, the structure of the low-level network is shown in Fig. \ref{FIG:3}. The network receives the subgoals and pixel observation. The observation inputs into a a three-layer CNN, of which the output will be given to a one-layer LSTM \cite{,hochreiter1997long}. Then the output embedding of LSTM connects with the subgoal embedding (get from a one-layer FC network), then input into three-layers FC network, learning by multi-process A2C \cite{,mnih2016asynchronous} algorithm to train a stable policy, as shown in algorithm 1.

\begin{figure}[t]
	\centering
	\includegraphics[scale=.32]{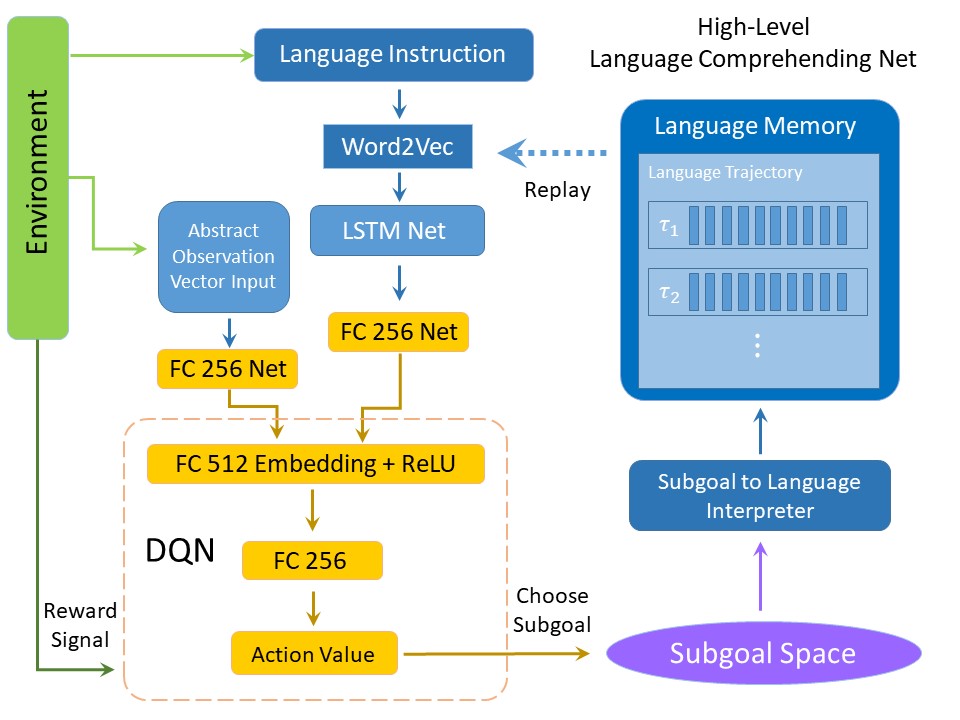}
	\caption{We build the high-level policy with DQN, which receives the abstract observation and the language instructions and we called language comprehending network (LCN). Also, the learned subgoals can be interpreted into learned language form, be stored in language memory buffer, and be reused in different tasks.}
	\label{FIG:4}
\end{figure}

\subsection{Building High-Level Language Comprehending Policy}

Here we show how to build abstract high-level policy, which we called language comprehending network (LCN).

After building a stable low-level policy, we intend to connect language with object-oriented subgoals. We describe the goal of each simple task with one complete sentence and just use RL method to learn to make a one-step decision in the same tasks with the low-level. 

For simplicity, the high-level observation is abstracted states similar to subgoals of the existing objects in a room. For instance, if the first position representing for a red box is ``2", it means that there are two red boxes in the current room. It is a way to control the low-level flexibly according to the change of observation instead of making a decision in a fixed step by giving subgoal every 5 steps like traditional HRL methods \cite{,vezhnevets2017feudal,nachum2018data}.

However, the high-level policy is not just a one-to-one interpreter from language to subgoals. That will limit the generalization abilities of the policy. We consider leveraging the fuzzy semantics of language. For instance, our policy also learns a fuzzy description such as ``open a door", which does not accurately express the goal. Then the ideal result is that the high-level policy gives the same probability among all the existing doors. 

In short, the high-level should receive language instruction and the abstracted observation then gives subgoals in sequence. The training tasks should lead the agent to a comprehension of the semantic information according to the observation instead of just interpreting language instruction. That will strongly increase the generalization abilities of the whole policy.  

Actually, we make use of \textit{nltk} \cite{,bird2006nltk} as a language preprocessing model to change every word into an embedding form Word2Vec \cite{,mikolov2013efficient}. The high-level network receives the language embedding vectors in an LSTM \cite{,hochreiter1997long} model. The observation is an abstracted vector said above representing existing objects from the original pixel state, which is used as input in FC network. Then we connect the output of the two networks as abstract input, build an FC network for training DQN \cite{,mnih2013playing} as shown in algorithm 2.

The output is a 24-dimension one-hot vector representing an object. The high-level policy makes decisions when the abstracted state changes or the low-level policy goes beyond the stated max step. When training, only when the agent correctly interacts with the appointed object, will it obtain the sparse reward 1, otherwise 0.

\begin{algorithm}[h] 
	\caption{High-level Task and Language Comprehending Policy Training Algorithm}\label{algorithm2}
	\begin{algorithmic}[1] 
		\STATE Initialize replay memory $\mathcal{D}$ to capacity $N$
		\STATE Initialize action-value $Q$ with random weight
		\STATE Initialize average success rate $s_r = 0$
		\STATE Set expected error rate $\varepsilon$
		
		\WHILE {$1 - s_r > \varepsilon$} 
		
		\STATE Sample instruction $I\in \mathcal{I}$ in uniform distribution
		\STATE With probability $\epsilon$ select a random subgoal \textit{existed in current observation} 
		\STATE Otherwise select subgoal $g_t = \max_g Q^*(O_t, g_t|I)$  with observation $O_t$
		\STATE Wait the low-level executing $g_t$ until success or   observation changed
		\STATE Store transition $(O_t,g_t,r_t,O_{t+1})$ in $\mathcal{D}$
		\STATE Sample mini-batch of transitions $(O_j,g_j,r_j,O_{j+1})$  from $\mathcal{D}$
		\STATE Set $$y_j  = \begin{cases} 
		r_j,  & \mbox{for terminal state } S_T \\
		r_j+\gamma \max \limits_{g_{j+1}}Q_{\theta^-}, & \mbox{otherwise}
		\end{cases}$$
		\STATE Update network parameters $\theta$ with
		$$d\theta \leftarrow d\theta + \frac{\partial}{\partial \theta} (y_j-Q(O_j,g_j;\theta))^2$$
		\STATE Calculate average success rate $s_r$ every 100  episodes  
		\ENDWHILE 
		
	\end{algorithmic} 
\end{algorithm}

\subsection{Constructing Extensible Policy With Abstract Language Trajectories}
Here we show how to build extensible compositional policy by an augmented-memory for generalization.

By combing the high-level and low-level policy, we obtain a policy that can correctly execute one sentence task such as ``pick up the red ball" in spite of environmental changing. However, when facing long-horizon tasks, which should be described by a section, the language input pattern is unseen and unrecognizable. That is also the problem that all the end-to-end frameworks facing, due to the limited generalization abilities \cite{,lake2018generalization}. For example, if training the agent by tasks that are described by less than three sentences but testing by the ones described by more than ten sentences, the end-to-end framework struggle to identify the content. Therefore, we design an additional structure to use language to avoid the problem.

We reiterate that we have already built stable object-oriented low-level policies. That means not only the language can be correctly executed by these policies, but also when executing the policy, it can be described by language unambiguously and correctly. The result is that, when the agent exploring by subgoal in a new unseen task, what the agent does can be output as the learned language by itself. These languages can be formed as the trained pattern. So if the agent stores these languages, it means that the agent can store trajectories in an abstracted language form, meanwhile can comprehend and reuse them. So the agent can explore a new task rapidly by low-dimension subgoal space and memorize the trajectory. Once the agent obtains the final sparse reward, it can solve the task by replay the language trajectories. If the new task has some randomness, the agent can also explore again by the abstracted trajectories. For example, if ``open the red door" is in the successful trajectory, the agent can heuristically explore by related words such as ``pick up a red object" or ``open a door". That is a way to generalize by abstract language space and build extensible policy.

Actually, we interpret every subgoal into language vector form in one sentence and store them into a memory buffer as the abstract language trajectories. When exploring a new task after training, the agent will explore by stochastic policy on subgoal space. When replaying language trajectory, the language will be taken out and be executed in sequence.

\section{Experiments}
We design experiments in stochastic and partial observation environment to show our method has these superiorities: 

(i) Our low-level policy can overcome random disturbance, correctly interacting with the object without other redundant actions. With the abstract subgoals and guidance of designed tasks, the low-level policy becomes a shared object-oriented policy which can execute the goal correctly.  

(ii) In several unseen new tasks without instruction guidance (gradually increasing the complexity), our agent can explore with object-oriented policies in the abstracted subgoal space, which is much smaller than the original state space, such that the agent can efficiently attain extremely sparse reward.

(iii) Once obtaining the reward by few-shot exploring, the agent can solve the task by replaying language trajectory memory and explore heuristically by language. Even new tasks with diversified randomness can be solved in a high probability. 

\subsection{Experiment Setting}

We choose BabyAI \cite{,chevalier-boisvert2018babyai} as our experiment platform. In this platform, there are a large number of various tasks which consist of many object-oriented tasks whose final goal will be described by structured synthetic natural language. These tasks are often generated with massive randomness, including random position, attribution, color, and other random disturbing objects. These tasks are all partially observed and long time horizon. They are difficult for the traditional RL method due to the frequent bottleneck state \cite{,mcgovern2001automatic} with sparse reward, which also forces the agent to identify the object and corresponding instructions. 

Based on the platform, we design a series of tasks, which consist of one or more $7 \times 7$ rooms. In every room, the agent will receive a pixel partial observation of the whole room, an instruction of natural language only describing the final goal of the task, and an abstracted observation, which is a vector representing the existing objects of the room. There are many objects in the room, some of which are the goal of the task and some are the disturbance. To go to the next room, the agent should open the correct door, or pick up the corresponding key of the door. Only when the agent achieves the final goal, will it get the sparse reward. We make use of these tasks to verify the superiority of our method.

The experiments setting are shown as follow:

\textbf{One Room Task :} As shown in Fig. \ref{FIG:ONE}, they are training tasks to interact with just one object in a single room. The tasks need agent to receive and comprehend the instruction given by the environment, meanwhile overcoming randomness and interacting with the correct object. In this task, we train our low-level policy with pixel observation and high-level policy with abstract observation respectively. The goals of the task consist of six colors and four shapes of objects, which build a multi-task joint training process.
\begin{figure}[h]
	\centering
	\includegraphics[scale=.5]{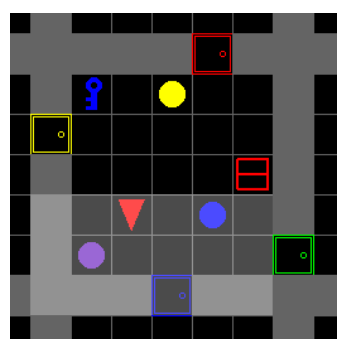}
	\caption{These tasks is used for training basic abilities of the agent. Every task has only one goal object.}
	\label{FIG:ONE}
\end{figure}

\textbf{Multiple Room Task :} As shown in Fig. \ref{FIG:THREE}, \ref{FIG:5_9}, they are many test tasks aiming to interact with a series of objects, including collecting keys and opening the right door, and entering in the next room until completing the task. In this task, there is \textit{no language instructions guidance}. We design the number of rooms of 3, 5, 9 with increasing difficulty. Comparing with the training task, these test tasks are all "out-of-domain". Every object in the room is fixed, but the position is random. Only the final object has a reward. Considering that our language memory can change and adapt to different tasks, the baselines are allowed to retrain in the tasks but ours' is not. These hard tasks test the compositional generalization abilities of the method in unseen tasks.
\begin{figure}[h]
	\centering
	\includegraphics[scale=.5]{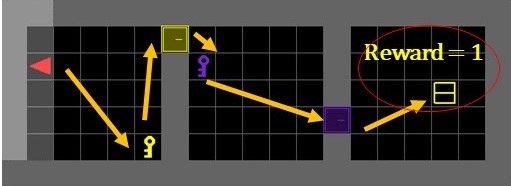}
	\caption{These tasks consist of three room with several keys and doors. The final room has a goal object with reward 1.}
	\label{FIG:THREE}
\end{figure}

\begin{figure}[t]
	\centering
	\includegraphics[scale=.5]{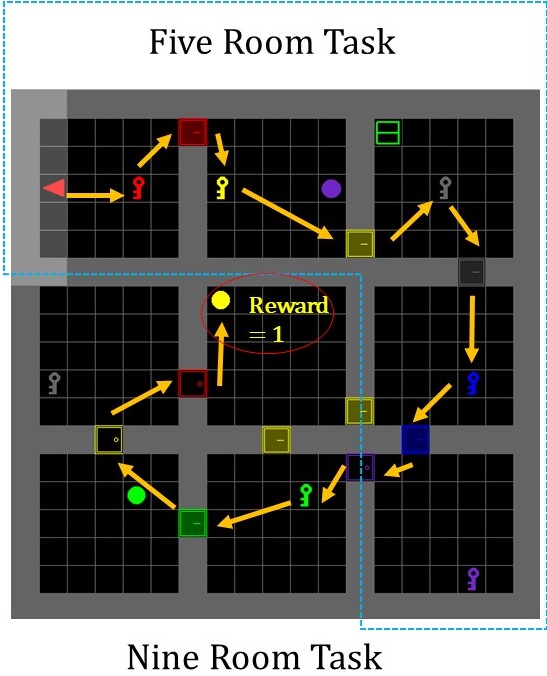}
	\caption{These tasks consist of many rooms with disturbing object. They require long-term exploration to obtain the final reward.}
	\label{FIG:5_9}
\end{figure}

\textbf{Multiple Room Task with Randomness:} This task is modified from the 9-room task above, where the objects are random besides the positions. For example, the last time there is a ``red door" on the wall, the next time there may be other doors, and the key to open the door changes with it. The task is to test whether the agent can not only generalize from compositional policies but also generalize from language space heuristically. Because the change of the object has a rule, which can be represented by language. Also, the task \textit{do not provide language instructions}. The extremely hard task shows superiorities of the method which can complete it.

\begin{figure*}
	\includegraphics[scale=.82]{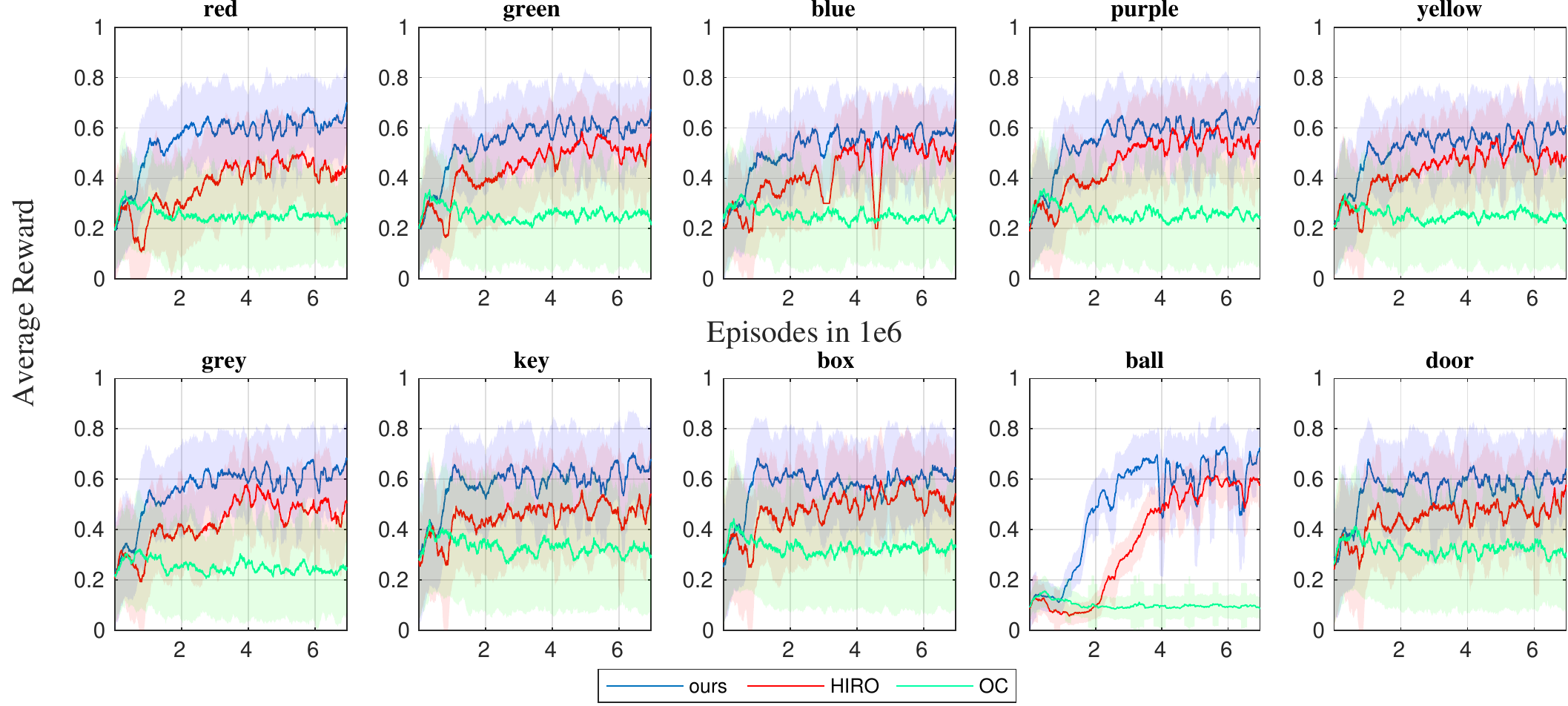}
	\caption{The figure is about the results of the one-room training experiment. They are divided into different kinds and shown according to colors and attributions in 10 classes. They are colors of \textbf{red, green, blue, purple, yellow, grey}, and shape of \textbf{key, box, ball, door}. Actually, the variances of the three shapes are influenced by the relationship of positions while the ball is not. So the result of the ball has lower variance.}
	\label{FIG:5}
\end{figure*}

\subsection{Baseline}

Here we will introduce the baseline. In our experiments, some of original methods lack the adaptive capabilities of language input and stochastic environment. We make experiments with both these methods and their modified version adapting to the task setting for fair comparison.  

The comparative baselines are shown as follow:

\textbf{Option-Critic.} It is a general and classical end-to-end hierarchical reinforcement learning method for many temporally extended tasks.\cite{,bacon2017the} These methods automatically build options by AC framework with learning, of which the option represents sub-policy for different subtasks. Considering that the method is originally designed without instruction, we design an LSTM  network the same to ours, to preprocess language input with \textit{nltk} as additional state. We should ensure that the information that other methods can get is equal to ours. For a fair comparison, we also modify the network and use a one-room task as the pre-training task to improve the performance of the baseline. 

In Table \ref{table}, they are ``OC4-ORI" for original OC4 method, ``OC4-INS" for OC4 adapting to language input, ``OC4-PRE" for OC4 with language and pre-training in ONE ROOM task.

\textbf{HIRO.} It is a data-efficient and general HRL method for long-horizon complex RL tasks \cite{,nachum2018data}. This method builds off-policy model-free RL framework with a correction to re-label the past experience. Same as OC, we also introduce natural language as complete task information, for identifying the final object. Also, we modified the network to adapt to our one-room pre-training task to improve the performance.

In table \ref{table}, they are ``HIRO-ORI" for original HIRO method, ``HIRO-INS" for HIRO adapting to language input, ``HIRO-PRE" for HIRO with language and pre-training in ONE ROOM task.

\textbf{Flat.} It is the low-level policy of our method without language instruction. The baseline is to show whether the task can be solved by RL \textit{without language}. 

\textbf{Traditional RL.} It is a baseline shows the capabilities of traditional RL methods. Including ``VANILLA-RL" for basic AC algorithm, ``STOCHASTIC" for stochastic policy, ``SHAPING" for reward shaping method. Many researchers deal with complex tasks by introducing additional rewards for the key objects or ``bottleneck state". Here we will show that in quite long time scale tasks, although agents can get some reward, the poor sampling efficiency of traditional RL will lead to failure and cannot achieve the final goal.

\subsection{Comparative Analysis in One-Room Basic Experiment}

For learning stable policy, we used an on-line 4-process A2C algorithm. That is, our method collects 4 trajectories and updates the network with an average gradient. That will reduce the variance brought by randomness. For a fair comparison, the baseline also uses multi-process training.

The result curves are shown in Fig. \ref{FIG:5}. In this experiment, our low-level policy compares with HIRO and OC methods. The result shows that our low-level policy can overcome the randomness and successfully learning these tasks gradually. OC and HIRO may not adapt to tasks with massive randomness, especially OC relies on the number of option and HIRO needs a correction for states as subgoals to learn policy efficiently. However, when the environment contains stochastic elements, that will make the correction a negative motivation. So that their performance in this task is a little weaker with a lower success rate.    

Our high-level policy is also pre-trained in this task, receiving abstracted observation and make one-step decision to give subgoal, of which the task is simplified. It is to learn identifying and decision-making abilities. Consider that the difficulty of the task is easier than the low-level policy and baselines, the curve will not be shown in the figure.

\begin{table*}
	\centering
	\caption{ Success percentage of generalizing tasks, where $<x$ means that success episodes is less than $x$ in 100 episodes. Traditional RL method are adequately tested in more than 500000 episodes. STEP is the minimum step to complete the task with randomness of position.}
	\scalebox{1.2}{
		\begin{tabular}{cccccc}
			\toprule
			\midrule
			&TASK & Three Room & Five Room & Nine Room & Random Nine Room  \\
			
			\multirow{3}{*}{\mbox{\shortstack[c]{\fontsize{8pt}{\baselineskip}\selectfont TASK \\ \fontsize{8pt}{\baselineskip}\selectfont SETTING }}} & STEPs-ORI-MIN &  20 $\sim$ 38  &  32 $\sim$ 68 &  44 $\sim$ 98 &44 $\sim$ 98  \\
			
			&STEPs-HIGH &     5     &     9     &     13     & 13
			\\
			
			&TEST EPISODE 	& 3000& 8000 & 15000 & 15000 \\
			
			\midrule
			&&\multicolumn{4}{c}{\centering SUCCESS PERCENTAGE (success / episodes \%)}  \\ \cmidrule(r){3-6}
			
			\multirow{3}{*}{\mbox{\shortstack[c]{\fontsize{8pt}{\baselineskip}\selectfont TRADITIONAL \\ \fontsize{8pt}{\baselineskip}\selectfont RL METHOD }}}&STOCHASTIC &$< 1e^{-3}$& $<1e^{-4}$&$<1e^{-4}$&$<1e^{-4}$ 
			\\
			&VANILLA-RL  &$< 1e^{-3}$& $<1e^{-4}$&$<1e^{-4}$&$<1e^{-4}$ 
			\\
			&SHAPING  &$< 1e^{-3}$& $<1e^{-4}$&$<1e^{-4}$&$<1e^{-4}$ 
			\\
			\midrule

			\multirow{3}{*}{\mbox{\shortstack[c]{\fontsize{8pt}{\baselineskip}\selectfont BASE \\ \fontsize{8pt}{\baselineskip}\selectfont LINE 1 }}}&HIRO-ORI &    $<1$   &$<1$        &  $<1$     &$<1$ 
			\\
			
			&HIRO-INS &  $<1$      & $<1$       &     $<1$   &$<1$ 
			\\
			&HIRO-PRE &    $ 39 \pm 17$    &   $ 5 \pm 10$    &    $<1$    &$<1$
			\\
			
			\midrule
			\multirow{3}{*}{\mbox{\shortstack[c]{\fontsize{8pt}{\baselineskip}\selectfont BASE \\ \fontsize{8pt}{\baselineskip}\selectfont LINE 2 }}}&OC4-ORI &    $<1$    &   $<1$     &  $<1$      &$<1$ 
			\\
			&OC4-INS &    $<1$   &   $<1$   &   $<1$     &$<1$ 
			\\
			&OC4-PRE &   $ 22 \pm 11$    &   $<1$     &    $<1$    & $<1$
			\\
			\midrule
			
			\multirow{2}{*}{\mbox{\shortstack[c]{\fontsize{8pt}{\baselineskip}\selectfont OUR \\ \fontsize{8pt}{\baselineskip}\selectfont METHOD }}}&FLAT &    $ 19 \pm 14$   & $ 1 \pm 5$ & $<1$ & $<1$
			\\
			& \textbf{LCN-SEN} &$ \textbf{79} \pm 14$&$\textbf{35} \pm 15$&$\textbf{ 28} \pm 13 $&$\textbf{ 6} \pm 8 $
			\\
			\midrule
			\bottomrule
		\end{tabular}
	}
	
	\label{table}
	
\end{table*}

\begin{figure*}
	\centering
	\includegraphics[scale=.5]{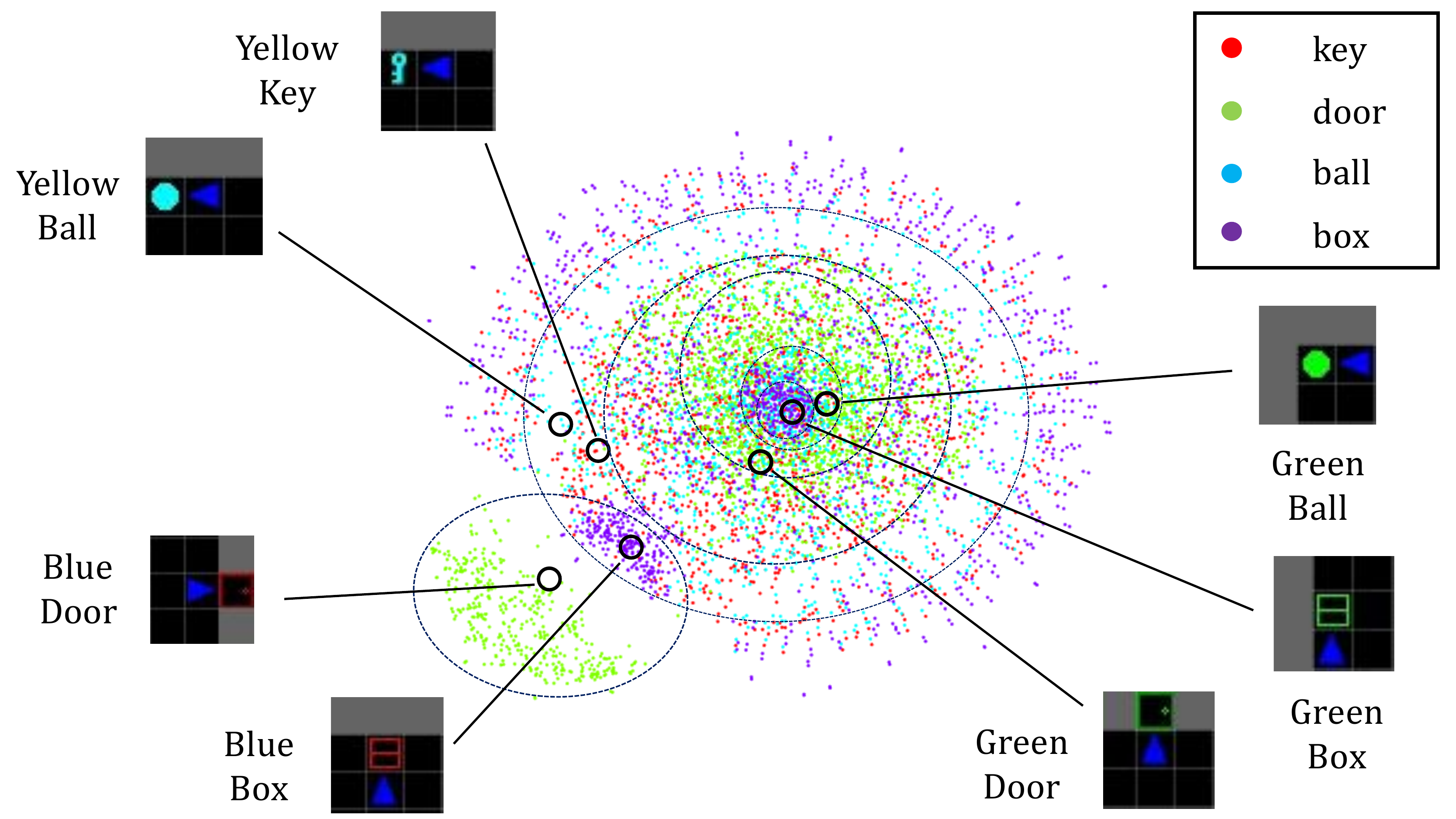}
	\caption{This is the a kind of visualization result of embedding of our network made by T-SNE method. The points of different classes are circular distributions. In this figure, we can see that the clustering result relies on many elements of the environment. Although some objects have different shape, the same color will made them closed. The figure shows that our method does learn some semantical information depending on the training process we designed.}
	\label{FIG:7}
\end{figure*}

\subsection{Generalization Experiment from Simple tasks to Complex Tasks}
These tasks are all long-horizon tasks. That means these tasks need the high-level policy of all the methods to take more than three-step decisions. Especially the 9-room task needs agent to interact with more than 15 different objects and takes more than 15 steps high-level decisions. These tasks are also extremely hard for most of the current RL methods due to the sparse final rewards and environment with randomness.

In these tasks, all the methods are allowed to retrain, except our methods. The one-room task will be seen as a pre-training tasks, and these task can be seen as OOD generalization tasks from the simple one to complex ones. Besides OC and HIRO, we add reward-shaping with traditional Deep RL methods (i.e., A2C), and OC and HIRO without pre-training. We want to verify the fact that, for complex tasks, although we can design an additional reward for shaping, the intricate relationship and the poor sampling efficiency will bring failure either. A frequent reward cannot always help to solve complex tasks.

Because only the final goal has a reward, the average reward can represents average success rate to some degree. The result shows that our method can rapidly explore in subgoal space and replay the language trajectory to complete the task in a high success probability.

\subsection{Visualization Experiment}
We leverage T-SNE \cite{,maaten2008visualizing} to show the semantical embedding learned by our training process as shown in Fig. \ref{FIG:7}. In this experiment, we visualize the embedding output of our network. We can see that our network does learned some semantical information with many kinds of different tasks. These ringlike clustering points has different semantics. Here we just show the result of different attributions. In fact all the elements, such as the shape, colors and the relative positions are all coupled, so all the labels mixed a little.

\subsection{Result of Random Long-Horizon Task}

The last experiment is shown for our strong ability to solve new difficult tasks with language replay buffer (See in Fig \ref{FIG:2} and Fig. \ref{FIG:5_9}). Facing such a hard task, our method still has a probability to complete the task (see in Table \ref{table}, the result of Random Nine Room). Our agent can explore in subgoal space rapidly even without the instruction of tasks as well as memorize the subgoal trajectory in a small buffer. Once getting the reward, a successful trajectory means the agent can explore the task with randomness by heuristic exploration. To the best of our knowledge, there is no method that can solve such a task with extremely sparse reward due to poor positive sample acquisition.

\section{Conclusion}
In this paper, we propose a new learning paradigm for building extensible and compositional language policy. Accordingly, we build a hierarchical RL policy with an invariable subgoal setting representing invariable semantics. These subgoals are stable and can be reused in any task of the environment. The two-level hierarchical model means that the agent can explore in low dimension semantical subgoal space. We also build an augmented-memory to record the trajectories of the agent by an abstract language form. It will help the agent generalize to new task by replaying language trajectories. 

However, the subgoal space of our method is fixed, which means we cannot generalize to completely new tasks, where the elements are unseen. And our method also has some limits which are caused by adapting to the environment. In future work, we will attempt to break through the shortage. Besides, we consider that if we can let the agent learn plenty of linguistic semantics, we can also teach it to think and make an inference by language, even interact with a human. Then it may build a general policy between quite different environments, such that one agent solves tasks in 2D and 3D space with similar semantics in the meantime, and can be guided by a human directly.

\ifCLASSOPTIONcaptionsoff
  \newpage
\fi



%

\bibliographystyle{IEEEtran}
\bibliography{main}

%

%
%
%




\end{document}